\documentclass[11pt]{article}

\usepackage{lineno}
\usepackage{graphicx, hyperref}

\modulolinenumbers[5]
\usepackage{color}
\usepackage{braket}
\usepackage{amsmath,amssymb}
\usepackage{mathrsfs}
\usepackage[top=30truemm,bottom=30truemm,left=30truemm,right=30truemm]{geometry}
\usepackage{indentfirst}
\usepackage{url}

\usepackage{wrapfig}
\usepackage{subcaption}

\usepackage{bm}
\usepackage{amsbsy}
\usepackage{amscd}
\usepackage{amsthm}
\usepackage{color}

\newtheorem{theorem}{Theorem}[section]
\newtheorem{lemma}[theorem]{Lemma}
\newtheorem{corollary}[theorem]{Corollary}

\newtheorem*{keywords}{\bf Keywords}
\newcommand{\RR}{\mathbb R}
\newcommand{\NN}{\mathbb N}

\newcommand{\EE}{\mathbb E}

\begin{document}

\title{Free energy of Bayesian Convolutional Neural Network with Skip Connection}

\author{Shuya Nagayasu and Sumio Watanabe\\
Department of Mathematical and Computing Science\\
Tokyo Institute of Technology,\\
Mail-Box W8-42, 2-12-1, Oookayama, \\
Meguro-ku, Tokyo,
152-8552, Japan
}

\date{}

\maketitle

\begin{abstract}
Since the success of Residual Network(ResNet), many of architectures of Convolutional Neural Networks(CNNs) have adopted skip connection. While the generalization performance of CNN with skip connection has been explained within the framework of Ensemble Learning, the dependency on the number of parameters have not been revealed. In this paper, we show that Bayesian free energy of Convolutional Neural Network both with and without skip connection in Bayesian learning. The upper bound of free energy of Bayesian CNN with skip connection does not depend on the oveparametrization and, the generalization error of Bayesian CNN has similar property.
\end{abstract}
\begin{keywords}
Learning theory; Convolutional Neural Network; Bayesian Learning; Free Energy 
\end{keywords}

\section{Introduction}
Convolutional Neural Networks (CNNs) are a type of Neural Networks mainly used for computer vision. CNNs have been shown the high performance with deep layers \  \cite{Szegedy2015, Krizhevsky2017}. Residual Network(ResNet) \ \cite{He2016} adopted the skip connection for addressing the problem that the loss function of CNN with deep layers does not decrease well through optimization. After success of ResNet, the CNNs with more than 100 layers are realized. The high performance of ResNet has been explained by similarity to the ensemble learning \ \cite{Huang2018, Nitanda2020, Ganaie2022}. On the other hand, there is a common issue in neural networks that the reason why the overparametrized deep neural network generalized has been unknown yet. 

In conventional learning theory, if the Fisher information matrix of a learning machine is positive definite, and the data size is sufficient large, the generalization error of the learning machine is determined from the number of its parameter in maximum likelihood estimator \ \cite{Akaike1974}. The similar property is shown in free energy and generalization error in Bayesian learning \ \cite{Schwarz1978,Rissanen1978, Akaike1980}. From these characteristics of generalization error and free energy some information criteria such as AIC, BIC, MDL are proposed. However, most of the hierarchical models such as neural networks have degenerated Fisher information matrix. In such models, the Bayesian generalization error and free energy are determined by a rational number called Real Log Canonical Threshold(RLCT) and that is smaller than the number of parameters \ \cite{Watanabe2001b, Watanabe2007}. In particular, RLCTs are revealed in some concrete models such as three layered neural networks \ \cite{Watanabe2001a, Aoyagi2012}, normal mixtures \ \cite{Hartigan1985, Yamazaki2003}, Poisson mixtures \ \cite{Sato2019}, Boltzmann machine \ \cite{Yamazaki2005, Aoyagi2010}, reduced rank regression \ \cite{Aoyagi2005}, Latent Dirichlet allocation \ \cite{Hayashi2021}, matrix factorization, and Bayesian Network \ \cite{Yamazaki2012}. While RLCTs of many hierarchical models are revealed, that of neural networks with multiple layer of nonlinear transformation has not been clarified. Yet the possibility of that is shown in \ \cite{Wei2022}, the RLCT of Deep Neural Network is revealed \ \cite{Nagayasu2023a}. On the other hand the RLCT of neural networks other than DNN was not explored.

In Bayesian learning for neural networks, how to realize the posterior is important. There exist approaches for generating posterior, Variational Approximation or Markov chain Monte Carlo(MCMC) methods. Variational Approximation for neural netowrks, Variational Autoencoder \ \cite{Kingma2013} or Monte Carlo dropout \ \cite{Gal2016} are practically used. Also for CNNs, variational approach for Bayesian inference was proposed \ \cite{Gal2015}. MCMC for neural networks, Hamiltonian Monte Carlo or Langevin Dynamics are useful for sampling from posterior. Stochastic Gradient Langevin Dynamics(SGLD) \ \cite{Welling2011} is a MCMC method applying Stochastic Gradient Descent instead of Gradient Descent to Langevin Dynamics is popular MCMC for Bayesian Neural Networks.  \cite{Zhang2019} \ used SGLD for generating posterior of CNNs.

In this paper we clarify the free energy and generalization error of Bayesian CNNs with and without skip connection. In both case the free energy and generalization error don't depend on the number of parameters in redundant filters. Then,  in case with skip connection, the redundant layers affect the free energy and generalization error whereas they don't affect in case without skip connection. This paper consists of seven main sections and one appendix. In section\ref{section:cnn}, we describe the setting of Convolutional Neural Network analyzed in this paper. In section\ref{section:Bayes}, we explain the basic terms of the Bayesian learning.  In section\ref{section:theorem}, we note the main theorem of this paper.  In section\ref{section:experiment}, we conducts the experiment of synthetic data.  In section\ref{section:discussion} and section\ref{section:conclusion}, we discuss about the theorem in this paper and conclusion.  In appendix\ref{apd:first}, we prove the main theorem of this paper.

\section{Convolutional Neural Network}\label{section:cnn}
 In this section we describe the function of Convolutional Neural Network. First, we explain CNN without skip connection. The kernel size is $3 \times 3$ with zero padding and 1-stride. The activation function is ReLU. The numbers of the layers of the CNN are $K_1 (\geq 3)$ for Convolutional Layers and $K_2 (\geq 3)$ for Fully Connected Layers. 

 Let $x \in \RR^{L_1 \times L_2 \times  H_1}$ be an input vector generated from $q(x)$ with bounded support and $y \in \NN$ be an output vector with $q(y|x)$ whose support is $\{1, \ldots ,H_{K_1 + K_2}\}$. We define $w^{(k)} \in \RR^{3 \times 3 \times H_{k - 1} \times H_{k}}$, $b^{(k)} \in \RR^{H_{k}}$ as weight and bias parameters in each Convolutional Layer $(2 \leq k \leq K_1)$. $f^{(k)} \in \RR^{L_1 \times L_2 \times H_k}$ is output of each layer for $1 \leq k \leq K_1$.  $\mathrm{Conv}(f,w)$ is the convolution operation with zero padding and 1-stride: 

\begin{align}
\mathrm{Conv}(f^{k - 1} , w^{k})_{l_1,l_2,h_{k}} = \sum_{h_{k-1}} \sum_{p=1,q=1}^{p=3,q=3} f_{l_1 + p -1, l_2 + p -1,h_{k - 1}} w_{p,q,h_{k- 1},h_{k}}.
\end{align}
We define $g(b^{k}): \RR^{H_k} \rightarrow \RR^{L_1 \times L_2 \times H_k}$ as

\begin{align}
g(b^{(k)})_{l_1,l_2} = b^{(k)} \;\;\; 
\end{align}
for $1 \leq l_1 \leq L_1, 1 \leq  l_2 \leq L_2$. By using $w^{(k)}$ , $g(b^{(k)})$, and $f^{(k-1)}$, $f^{(k)}$ is described by 

\begin{align}
f^{(k)}(w,b,x) = \sigma(\mathrm{Conv}(f^{(k-1)}(w,b,x),w^{(k)}) + g(b^{(k)})) 
\end{align}
where $w, b$ are the set of all weight and bias parameters. $\sigma()$ is a function that applies the ReLU to all the elements of the input tensor. 

The output of $k = K_1 + 1$  layer is result of Global Average Pooling on $k = K_1$ layer:
\begin{align}
f^{(K_1 + 1)}(w,b,x) = \frac{1}{L_1L_2} \sum_{l_1= 1}^{l_1= L_1} \sum_{l_2 = 1}^{l_2 = L_2} f^{(K_1)}(w,b,x)_{l_1,l_2}.
\end{align}

Let $w^{(k)} \in \RR^{H_{k}} \times \RR^{H_{k - 1}}$, $b^{(k)} \in \RR^{H_{k}}$ be weight and bias parameters in each Fully Connected Layer $(K_1 + 2 \leq k \leq K_1 + K_2)$. 
For $K_1 + 2 \leq k \leq K_1 + K_2 - 1$, $f^{(k)}$ is defined by
\begin{align}
f^{(k)}(w,b,x) = \sigma(w^{(k)}f^{(k - 1)}(w,b,x) + b^{(k)}),
\end{align}
and for $k = K_1 + K_2$,
\begin{align}
f^{(K_1 + K_2)}(w,b,x) = \mathrm{softmax}(w^{(k)}f^{(k - 1)}(w,b,x) + b^{(k)}),
\end{align}
where $\mathrm{softmax}()$ is a softmax function
\begin{align}
\mathrm{softmax}(z)_i = \frac{e^{z_i}}{\sum_{j = 1}^{J}e^{z_j}}.
\end{align}
The output of the model is represented stochastically 
\begin{align}
y \sim \mathrm{Categorical}(f^{(K_1 + K_2)}(w,b,x))
\end{align}
where $\mathrm{Categorical}()$ is a categorical distribution. 

Then we describe CNN with skip connection. The number of layers within the skip connection is $K_s$ and the number of skip connection is $M$.  The output of the layer with skipped connection is described by
\begin{align}
f^{(mK_s + 2)}(w,b,x) = \sigma(\mathrm{Conv}(&f^{(mK_s + 1)}(w,b,x),w^{(mK_s + 2)})  \nonumber \\  
                                            &+ B^{(mK_s + 2)} +  f^{((m - 1)K_s + 2)}(w,b,x)). 
\end{align}
In this case, CNN satisfies the following conditions
\begin{align}
K_1 &= MK_s + 2 \nonumber \\
H^{mK_s + 2} &= \mathrm{const} (1 \leq m \leq M) \label{eq:cond3}.
\end{align}
The other conditions are the same as the case without skip connection. 

\begin{figure}[htbp]
  \begin{minipage}{6cm}
  \begin{center}
    \includegraphics[keepaspectratio, scale=0.13]{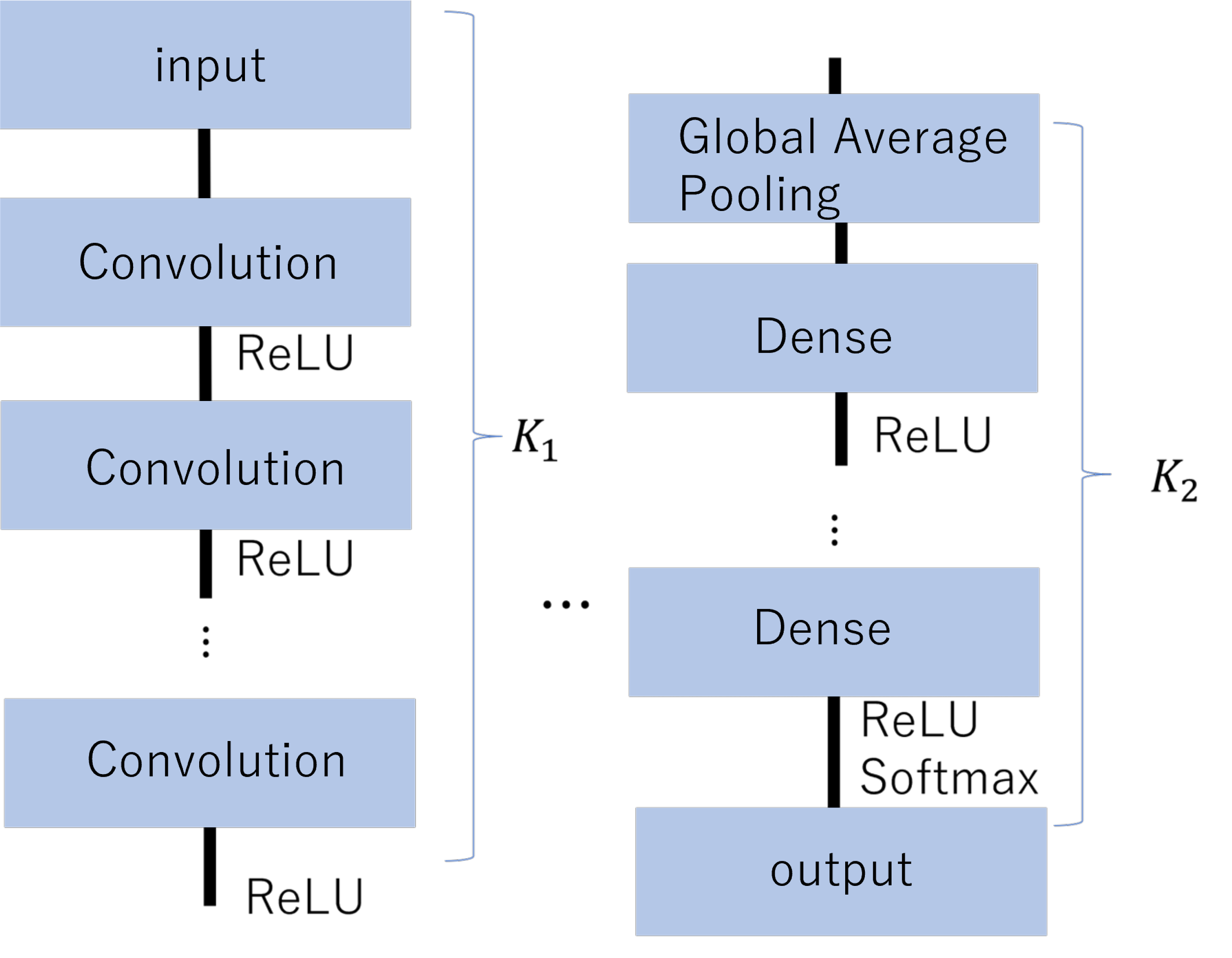}
    No Skip Connection
  \end{center}
  \end{minipage}
  \begin{minipage}{5cm}
  \begin{center}
    \includegraphics[keepaspectratio, scale=0.13]{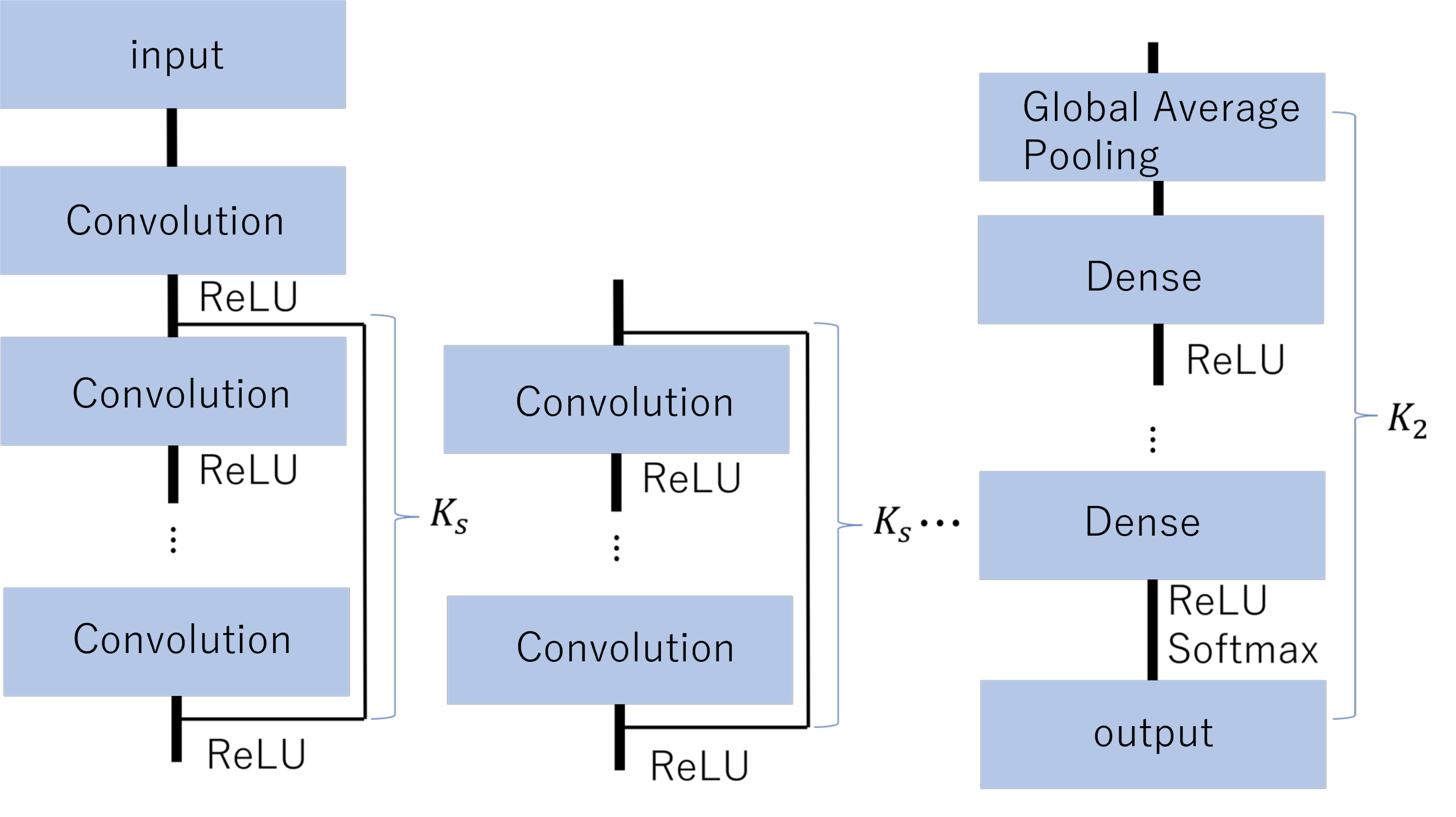}
    Skip Connection
    \end{center}
  \end{minipage}
  \caption{The structure of Convolutional Neural Network with and without Skip Connection}
  \label{fig_model}
\end{figure}
Figure\ref{fig_model} shows the configuration of Convolutional Neural Network analyzed in this paper. 

\section{Free energy in Bayesian Learning}\label{section:Bayes}
\subsection{Bayesian Learning}
Let $X^n = (X_1, \cdots X_n)$ and $Y^n = (Y_1, \cdots Y_n)$ be training data and labels. $n$ is the number of the data. These data and labels are generated from a true distribution $q(x,y) = q(y|x)q(x)$. The prior distribution $\varphi(w)$, the learning model $p(y|x,w)$ is given on the bounded parameter set $W$. Then the posterior distribution is defined by 
\begin{align}
p(w|X^n,Y^n) = \frac{1}{Z(Y^n|X^n)}\varphi(w)\prod_{i=1}^{n}p(Y_i|X_i, w)
\end{align}
where $Z_n = Z(Y^n|X^n)$ is normalizing constant denoted as marginal likelihood: 
\begin{align}
Z_n = \int \varphi(w)\prod_{i=1}^{n}p(Y_i|X_i,w)\mathrm{d}w.
\end{align}
The free energy is negative log value of marginal likelihood 
\begin{align}
F_n = - \log Z_n.
\end{align}
Free energy is equivalent to evidence and stochastic complexity. 

The posterior predictive distribution is defined as the average of the model by posterior: 
\begin{align}
p^{*}(y|x) = p(y|x , X^n,Y^n) = \int p(y|x,w) p(w|X^n,Y^n)\mathrm{d}w. 
\end{align}
Generalization error $G_n$ is given by Kullback-Leibler divergence between the true distribution and posterior distribution as follows  
\begin{align}
G_n = \int q(y|x)q(x) \log \frac{q(y|x)}{p^{*}(y|x) } \mathrm{d}x \mathrm{d}y.
\end{align} 
Average of Generalization error is difference between the average of Free energy of $n$ and $n + 1$:
\begin{align}
\EE[G_n] - S = \EE[F_{n+1}] - \EE[F_n], \label{eq:fngn}
\end{align} 
where $\EE[f(X^n,Y^n)]$ is the average of the generation of n data $\EE_{X^n,Y^n}[f(X^n,Y^n)]$. 

\subsection{Asymptotic property of Free energy and Generalization error}

It is well known that if the average Kullback-Leibler divergence 
\begin{align}
K(w) = \int q(y|x)q(x) \log \frac{q(y|x)}{p(y|x,w)} \mathrm{d}x \mathrm{d}y.
\end{align}
can be approximated by quadratic form, in other words, the Laplace approximation can be applied to the posterior distribution, average of Free energy has the following asymptotic expansion with the number of parameters of learning model$d$ \ \cite{Schwarz1978, Rissanen1978} 
\begin{align}
E[F_n] = n(S + \mathrm{Bias}) + \frac{d}{2}\log n + O(1)
\end{align}
where $S$ is entropy of true distribution and $\mathrm{Bias}$ is the minimum value of $K(w)$ for $w \in W$. The generalization error is calculated from Free energy by using equation(\ref{eq:fngn}) \ \cite{Akaike1980}:
\begin{align}
E[G_n] = \mathrm{Bias} + \frac{d}{2n} + o\left(\frac{1}{n}\right). 
\end{align}
Laplace approximation cannot be applied to the average Kullback-Leibler divergence of hierarchical model such as Gaussian Mixture or neural networks  because of the degeneration of Fisher information matrix. In such models, the average of Free energy and Generalization error have the following asymptotic expansions \ \cite{Watanabe2001b}
\begin{align}
E[F_n] &= n(S + \mathrm{Bias}) + \lambda \log n + o(\log n), \\
E[G_n] &= \mathrm{Bias} + \frac{\lambda}{n} + o\left(\frac{1}{n}\right),  \label{eq:gerror}
\end{align}
where $\lambda$ is a rational number called Real Log Canonical Threshold(RLCT). In particular, \cite{Nagayasu2023a} showed that in case $\mathrm{Bias} = 0$ and $x$ is bounded, when the Deep Neural Network is trained from the data generated from smaller network,  
\begin{align}
\lambda \leq \frac{d^*}{2}
\end{align}
where $d^* \leq d$ is the number of parameter of data generating Network. 
 

\section{Main Theorem}\label{section:theorem}

In this subsection the main result of this paper is introduced. First, to state the theorem, we define the data generating network. Both in skip connection the data generating network satisfies the following conditions about the number of layers and filters,
\begin{align}
K^*_1 \leq K_1, K^*_2 \leq K_2, (H^*)^{(1)} = H^{(1)} (H^*)^{(K_1)} = H^{(K_1 + K_2)} \label{eq:cond1}
\end{align}
and 
\begin{align}
H^{(k)} &\geq (H^*)^{(K^*_1)} (K^*_1 + 1 \leq k \leq K_1) \nonumber \\
H^{(k)} &\geq (H^*)^{(K_1 + K^*_2)} (K_1 + K^*_2 + 1 \leq k \leq K_1 + K_2 - 1) \nonumber \\
H^{(k)} &\geq (H^*)^{(k)} (others)  \label{eq:cond2}.
\end{align}
Then, we show the main theorem.
\begin{theorem} (No Skip connection)\label{theorem:111}
Assume that the learning machine and the data generating distribution 
are given by $p(y|x,w,b)$ and $q(y|x)=p(y|x,w^*,b^*)$ in case without skip connection which satisfy the conditions
(\ref{eq:cond1}) and (\ref{eq:cond2}), and that a training data 
$\{(X_i,Y_i)\;\;i=1,2,...,n\}$ is independently taken from $q(x)q(y|x)$. Then 
the average free energy satisfies the inequality, 
\begin{align}\label{eq:lambda1}
\EE[F_n]\leq nS+ \lambda_{CNN} \log n +C
\end{align}
where
\begin{align}\label{eq:lambda1}
\lambda_{CNN}=\frac{1}{2}
\left(|w^*|_0 + |b^*|_0 + \sum_{k = K^*_1 + 1}^{K_1}(9H_{K^*_1} + 1)H_{K^*_1}
\right)
\end{align}
where $|w^*|_0 , |b^*|_0$ are the numbers of parameters of weights and biases in data generating network. 
\end{theorem}

\begin{theorem} (Skip connection)\label{theorem:112}
Assume that the learning machine and the data generating distribution 
are given by $p(y|x,w,b)$ and $q(y|x)=p(y|x,w^*,b^*)$ in case with skip connection which satisfy the conditions (\ref{eq:cond3}), (\ref{eq:cond1}) and (\ref{eq:cond2}), and that a training data 
$\{(X_i,Y_i)\;\;i=1,2,...,n\}$ is independently taken from $q(x)q(y|x)$. Then 
\begin{align}\label{eq:lambda2}
\lambda_{CNN}=\frac{1}{2}(|w^*|_0 + |b^*|_0)
\end{align}
\end{theorem}

Proof of main theorems are shown in Appendix\ref{apd:first}.

If there exists asymptotic expansion of the generalization error $\EE[G_n]$ in theorem\ref{theorem:111} and theorem\ref{theorem:112}, that satisfies the following inequality  
\begin{align}
\EE[G_n] \leq \frac{\lambda_{CNN}}{n} + o\left(\frac{1}{n}\right),
\end{align}
where 
\begin{align}
G_n = \int q(x)\sum_{i = 1}^{H_{K_1 + K_2}} f_i^{(K^*_1 + K^*_2)}(w^*,b^*,x) \log \frac{f_i^{(K^*_1 + K^*_2)}(w^*,b^*,x)}{\EE_{w,b}[f_i^{(K_1 + K_2)}(w,b,x)]} \mathrm{d}x
\end{align}
which corresponds to categorical cross entropy.

\section{Experiment}\label{section:experiment}
In this section, we show the result of experiment of synthetic data. 

\subsection{Methods}
We prepared the 2-class labeled simple data shown in fig\ref{fig_synthetic}. The the data is $x \in R^{4 \times 4 \time 1}$ and the values of each elements are in $(-1,1)$. The average of each element is $0.5$ or $-0.5$ and added the truncated normal distribution noise within the interval$(-0.5,0.5)$. The probability of each label of data is 0.5. We trained CNN whose number of convolutional layer $K_1 = 2$ and fully connected layers$K_2 = 2$ with SGD. The number of filter is  $H_2 = 2$ and the parameters are $L_2$ regularized. We use the trained CNN named "true model" as a data generating distribution. Note that the label of original data fig\ref{fig_synthetic} is deterministic but the label of true model is probabilistic. We prepare three learning CNN models. Each number of convolutional layers is $K_1 = 2,3,4$. Each model has skip connection every one layers or does not have skip connection. The number of filters in each layers is $H^{(k)} = 4$. They have $K_2 = 2$ fully connected layers. The prior distribution is the Gaussian distribution which covariance matrix is $10^4 I$ for weight parameter and $10^2 I$ for bias parameter. 
We train the learning CNN models by using the Langevin dynamics. The learning rate is $10^{-2}$ and  the interval of sampling is 100. We use the average of $1000$ samples of learning CNN models as the average of posterior. We estimate the generalization error by the test error of $10000$ test data from true model. We trained each learning model $10$ times and estimated the $\EE[G_n]$ from the average of test error.  

\begin{figure}[htbp]
  \begin{minipage}{10cm}
  \begin{center}
    \includegraphics[keepaspectratio, scale=0.5]{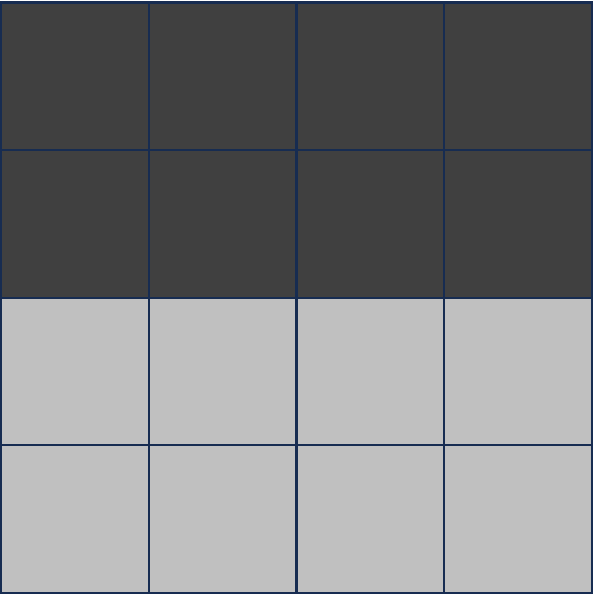}
  \end{center}
  \end{minipage}
  \begin{minipage}{5cm}
  \begin{center}
    \includegraphics[keepaspectratio, scale=0.5]{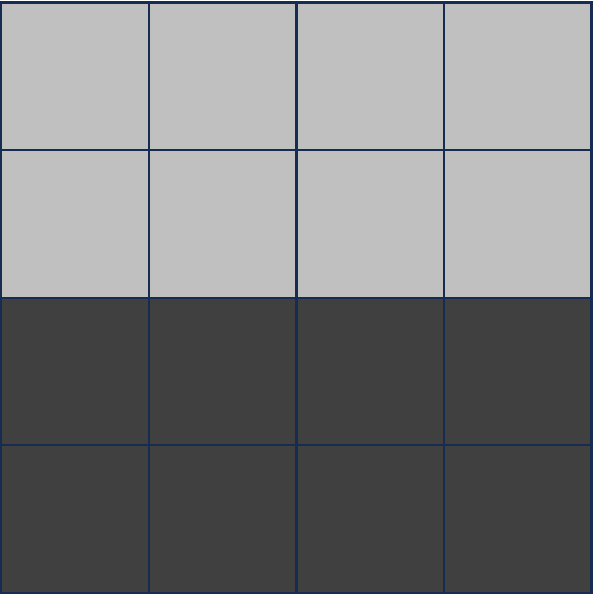}
    \end{center}
  \end{minipage}
  \caption{The average of input x of each label}
  \label{fig_synthetic}
\end{figure}

\subsection{Result of experiments}
\begin{table}[h]
 \caption{Experimental value and theoretical upper bound of the generalization error}
 \label{table:result}
 \centering
  \begin{tabular}{ccccc}
   \hline
   model &  $n \times$ Test Error &  $\lambda_{\mathrm{CNN}}$  &   $d_{\mathrm{model}} / 2$ \\
   \hline \hline
   $K_1 = 2$           & $ 16.0(1.9) $ & $   13  $ & $  25  $   \\
   $K_1 = 3$ no skip & $ 10.0(0.9) $ & $   32  $ & $  99  $   \\
   $K_1 = 4$ no skip & $ 58.4(2.3) $ & $   51  $ & $ 173  $   \\
   $K_1 = 3$ with skip& $ 11.4(2.3) $ & $  13 $ & $   99  $   \\
   $K_1 = 4$ with skip& $ 15.6(1.2) $ & $  13  $ & $ 173  $   \\
   \hline
  \end{tabular}
\end{table}
Table\ref{table:result} shows the result of the experiment. Test Error shows n times of the average of 10 test error in each model and the standard error of them.   $d_{\mathrm{model}}$ is a number of parameters of each model. All the CNN models include the true model, hence the bias is $0$. Then from equation(\ref{eq:gerror}), theoretical upper bound of the generalization error is $\lambda_{\mathrm{CNN}} / n$. In table\ref{table:result}, the experimental values of all models are smaller than $d_{\mathrm{model}} / 2$. 
Moreover in case with skip connection, the experimental value did not so increase with the increase of the number of layer. Then, in case $K_1 = 4$ without skip connection, the experimental value increased from the case$K_1 = 2$.  In case $K_1 = 3$ without skip connection, the experimental value is smaller than that of $K_1 = 2$. Behavior of MCMC is considered to be the cause of this result. Since MCMC in high dimensional model needs the long series for convergence in general, the result is deviated from theoretical predict.

\section{Discussion}\label{section:discussion}
\subsection{Difference with or without Skip Connection}
In this paper for analyzing the overparametrized CNN, the data generating network is smaller than learning network both case of Skip Connection. Nevertheless two cases of the data generating network is different, if the learning model network has double filter $H^{(k)}$ to the data generating network in each Convolutional Layer, the model network can represent the generating network in different case. The output of each layer is nonnegative hence the model can represent the skip connection or the negative of that. If the model network doesn't have larger layer to the data generating network, the free energy of CNN with skip connection can be both larger or smaller than that without skip connection by the data generating network. Then, the layer of model network gets larger, the free energy of CNN with skip connection does not change but that without skip connection gets larger and the free energy of CNN with skip connection comes to have smaller free energy for all data generating network.

\subsection{Comparison to Deep Neural Network}
Firstly we compare the result of this paper to that of DNN in \ \cite{Nagayasu2023a}. In case of DNN, the free energy depends on the layers of the model and only on that of the data generating network. This stands to the reason that mapping of the linear transformation in lower layer can be represented in higher layer. On the other hand, convolution operation doesn't have such property hence, the free energy of CNN without skip connection depends on the layer of learning model network. However, with skip connection, there exists the essential parameter which doesn't depend on overparametrized layers and the free energy does not also depend on the layer of learning model network.

\section{Conclusion}\label{section:conclusion}

In this paper, we studied Free energy of Bayesian Convolutinal Neural Network with Skip Connection 
and compared to the case without Skip Connection. Free energy of Bayesian CNN with Skip Connection doesn't depend on the layer of the model unlike the case without Skip Connection.  In Bayesian learning, the increase of Free energy is equivalent to generalization error, hence the generalization error has same property about the Skip Connection. In particular, Free energy of CNN without skip connection does not depends on the number of parameters in learning network but depends only on that in data generating network. This feature shows the generalization ability of CNN with skip connection does not decrease with respect to any overparameterization in Bayesian learning.      


\bibliography{ref}

\begin{thebibliography}{10}
\expandafter\ifx\csname url\endcsname\relax
  \def\url#1{\texttt{#1}}\fi
\expandafter\ifx\csname urlprefix\endcsname\relax\def\urlprefix{URL }\fi
\expandafter\ifx\csname href\endcsname\relax
  \def\href#1#2{#2} \def\path#1{#1}\fi

\bibitem{Szegedy2015}
C.~Szegedy, W.~Liu, Y.~Jia, P.~Sermanet, S.~Reed, D.~Anguelov, D.~Erhan,
  V.~Vanhoucke, A.~Rabinovich, Going deeper with convolutions, in: Proceedings
  of the IEEE conference on computer vision and pattern recognition, 2015, pp.
  1--9.

\bibitem{Krizhevsky2017}
A.~Krizhevsky, I.~Sutskever, G.~E. Hinton, Imagenet classification with deep
  convolutional neural networks, Communications of the ACM 60~(6) (2017)
  84--90.

\bibitem{He2016}
K.~He, X.~Zhang, S.~Ren, J.~Sun, Deep residual learning for image recognition,
  in: Proceedings of the IEEE conference on computer vision and pattern
  recognition, 2016, pp. 770--778.

\bibitem{Huang2018}
F.~Huang, J.~Ash, J.~Langford, R.~Schapire, Learning deep resnet blocks
  sequentially using boosting theory, in: International Conference on Machine
  Learning, PMLR, 2018, pp. 2058--2067.

\bibitem{Nitanda2020}
A.~Nitanda, T.~Suzuki, Functional gradient boosting for learning residual-like
  networks with statistical guarantees, in: International Conference on
  Artificial Intelligence and Statistics, PMLR, 2020, pp. 2981--2991.

\bibitem{Ganaie2022}
M.~A. Ganaie, M.~Hu, A.~Malik, M.~Tanveer, P.~Suganthan, Ensemble deep
  learning: A review, Engineering Applications of Artificial Intelligence 115
  (2022) 105151.

\bibitem{Akaike1974}
H.~Akaike, A new look at the statistical model identification, IEEE
  transactions on automatic control 19~(6) (1974) 716--723.
\newblock \href {https://doi.org/10.1109/tac.1974.1100705}
  {\path{doi:10.1109/tac.1974.1100705}}.

\bibitem{Schwarz1978}
G.~Schwarz, Estimating the dimension of a model, The annals of statistics 6~(2)
  (1978) 461--464.
\newblock \href {https://doi.org/10.1214/aos/1176344136}
  {\path{doi:10.1214/aos/1176344136}}.

\bibitem{Rissanen1978}
J.~Rissanen, Modeling by shortest data description, Automatica 14~(5) (1978)
  465--471.
\newblock \href {https://doi.org/10.1016/0005-1098(78)90005-5}
  {\path{doi:10.1016/0005-1098(78)90005-5}}.

\bibitem{Akaike1980}
H.~Akaike, Likelihood and the bayes procedure, in: Springer Series in
  Statistics, Springer New York, 1998, pp. 309--332.
\newblock \href {https://doi.org/10.1007/978-1-4612-1694-0_24}
  {\path{doi:10.1007/978-1-4612-1694-0_24}}.

\bibitem{Watanabe2001b}
S.~Watanabe, Algebraic geometrical methods for hierarchical learning machines,
  Neural Networks 14~(8) (2001) 1049--1060.
\newblock \href {https://doi.org/10.1016/s0893-6080(01)00069-7}
  {\path{doi:10.1016/s0893-6080(01)00069-7}}.

\bibitem{Watanabe2007}
S.~Watanabe, Almost all learning machines are singular, in: 2007 IEEE Symposium
  on Foundations of Computational Intelligence, IEEE, 2007, pp. 383--388.

\bibitem{Watanabe2001a}
S.~Watanabe, Algebraic analysis for nonidentifiable learning machines, Neural
  Computation 13~(4) (2001) 899--933.
\newblock \href {https://doi.org/10.1162/089976601300014402}
  {\path{doi:10.1162/089976601300014402}}.

\bibitem{Aoyagi2012}
M.~Aoyagi, K.~Nagata, Learning coefficient of generalization error in bayesian
  estimation and vandermonde matrix-type singularity, Neural Computation 24~(6)
  (2012) 1569--1610.
\newblock \href {https://doi.org/10.1162/neco_a_00271}
  {\path{doi:10.1162/neco_a_00271}}.

\bibitem{Hartigan1985}
J.~A. Hartigan, A failure of likelihood asymptotics for normal mixtures, in:
  Proceedings of the Barkeley Conference in Honor of Jerzy Neyman and Jack
  Kiefer, 1985, Vol.~2, 1985, pp. 807--810.

\bibitem{Yamazaki2003}
K.~Yamazaki, S.~Watanabe, Singularities in mixture models and upper bounds of
  stochastic complexity, Neural Networks 16~(7) (2003) 1029--1038.
\newblock \href {https://doi.org/10.1016/s0893-6080(03)00005-4}
  {\path{doi:10.1016/s0893-6080(03)00005-4}}.

\bibitem{Sato2019}
K.~Sato, S.~Watanabe, Bayesian generalization error of poisson mixture and
  simplex vandermonde matrix type singularity, arXiv preprint arXiv:1912.13289
  (2019).

\bibitem{Yamazaki2005}
K.~Yamazaki, S.~Watanabe, Singularities in complete bipartite graph-type
  boltzmann machines and upper bounds of stochastic complexities, {IEEE}
  Transactions on Neural Networks 16~(2) (2005) 312--324.
\newblock \href {https://doi.org/10.1109/tnn.2004.841792}
  {\path{doi:10.1109/tnn.2004.841792}}.

\bibitem{Aoyagi2010}
M.~Aoyagi, A bayesian learning coefficient of generalization error and
  vandermonde matrix-type singularities, Communications in Statistics - Theory
  and Methods 39~(15) (2010) 2667--2687.
\newblock \href {https://doi.org/10.1080/03610920903094899}
  {\path{doi:10.1080/03610920903094899}}.

\bibitem{Aoyagi2005}
M.~Aoyagi, S.~Watanabe, Stochastic complexities of reduced rank regression in
  bayesian estimation, Neural Networks 18~(7) (2005) 924--933.
\newblock \href {https://doi.org/10.1016/j.neunet.2005.03.014}
  {\path{doi:10.1016/j.neunet.2005.03.014}}.

\bibitem{Hayashi2021}
N.~Hayashi, The exact asymptotic form of bayesian generalization error in
  latent dirichlet allocation, Neural Networks 137 (2021) 127--137.
\newblock \href {https://doi.org/10.1016/j.neunet.2021.01.024}
  {\path{doi:10.1016/j.neunet.2021.01.024}}.

\bibitem{Yamazaki2012}
K.~Yamazaki, S.~Watanbe, Stochastic complexity of bayesian networks, arXiv
  preprint arXiv:1212.2511 (2012).

\bibitem{Wei2022}
S.~Wei, D.~Murfet, M.~Gong, H.~Li, J.~Gell-Redman, T.~Quella, Deep learning is
  singular, and that's good, {IEEE} Transactions on Neural Networks and
  Learning Systems (2022) 1--14\href
  {https://doi.org/10.1109/tnnls.2022.3167409}
  {\path{doi:10.1109/tnnls.2022.3167409}}.

\bibitem{Nagayasu2023a}
S.~Nagayasu, S.~Watanabe, Bayesian free energy of deep relu neural network in
  overparametrized cases, arXiv preprint arXiv:2303.15739 (2023).

\bibitem{Kingma2013}
D.~P. Kingma, M.~Welling, Auto-encoding variational bayes, CoRR abs/1312.6114
  (2013).

\bibitem{Gal2016}
Y.~Gal, Z.~Ghahramani, Dropout as a bayesian approximation: Representing model
  uncertainty in deep learning, in: international conference on machine
  learning, PMLR, 2016, pp. 1050--1059.

\bibitem{Gal2015}
Y.~Gal, Z.~Ghahramani, Bayesian convolutional neural networks with bernoulli
  approximate variational inference, arXiv preprint arXiv:1506.02158 (2015).

\bibitem{Welling2011}
M.~Welling, Y.~W. Teh, Bayesian learning via stochastic gradient langevin
  dynamics, in: Proceedings of the 28th international conference on machine
  learning (ICML-11), 2011, pp. 681--688.

\bibitem{Zhang2019}
Y.~Zhang, S.~Pal, M.~Coates, D.~Ustebay, Bayesian graph convolutional neural
  networks for semi-supervised classification, in: Proceedings of the AAAI
  conference on artificial intelligence, Vol.~33, 2019, pp. 5829--5836.

\end{thebibliography}

\appendix

\section{Proof of main theorem}\label{apd:first}

In this Appendix, we show the proof of main theorem. 
\subsection{Inequalities}
Note that we describe the Frobenius norm of any order of tensor as $\|\cdots\|$.
We denote the Kullback-Leibler divergence of a data-generating distribution
$q(y|x)=p(y|x,w^*,b^*)$ and 
a model $p(y|x)$ that
\begin{align}
K(w,b)=\int q(x)q(y|x)\log\frac{q(y|x)}{p(y|x,w,b)}\mathrm{d}x\mathrm{d}y. 
\end{align}

\begin{lemma} \label{lemma:111} \cite{Nagayasu2023a}
Assume that a set $W$ is contained in the set determined by 
the prior distribution $\{(w,b);\varphi(w,b)>0\}$. Then for an 
arbitrary postive integer $n$, 
\begin{align}
\EE[F_n]\leq nS -\log \int_{W} \exp(-n K(w,b))\varphi(w,b) \mathrm{d}w \mathrm{d}b.
\end{align}
\end{lemma}

\begin{lemma}\label{lemma:222} \cite{Nagayasu2023a}
For  arbitrary vectors $s,t$, 
\begin{align}
\|\sigma(s)-\sigma(t)\|\leq \|s-t\|.
\end{align}
\end{lemma}

\begin{lemma} \label{lemma:333} \cite{Nagayasu2023a}
For arbitrary $w$,$w'$, $b$, $b'$, and $K_1 + 1 \leq  k \leq K_1 + K_2$, the following inequality holds, 
\begin{align}
& \|f^{(k)}(w,b,x)-
f^{(k)}(w',b',x) \|\nonumber
\\
&\leq \|w^{(k)}-w'^{(k)}\|\|f^{(k-1)}(w,b,x)\|
+  \|b^{(k)}-b'^{(k)}\|\nonumber
\\
&+ \|w^{(k)}\|\|f^{(k-1)}(w,b,x)-f^{(k-1)}(w',b',x)\|.
\end{align}
\end{lemma}

\begin{corollary} \label{corollary:333}
For arbitrary $w$,$w'$, $b$, $b'$,  and $1 \leq  k \leq K_1$,  the following inequality holds, 
\begin{align}
& \|f^{(k)}(w,b,x)-
f^{(k)}(w',b',x) \|\nonumber
\\
&\leq 9\|w^{(k)}-w'^{(k)}\|\|f^{(k-1)}(w,b,x)\|
+ L_1 L_2  \|b^{(k)}-b'^{(k)}\|\nonumber
\\
&+ 9\|w^{(k)}\|\|f^{(k-1)}(w,b,x)-f^{(k-1)}(w',b',x)\|\nonumber
\\
&+ \delta^{(k)}\|w^{(k)}\|\|f^{(k - K_2 -1)}(w,b,x)-f^{(k - K_2 -1)}(w',b',x)\|.
\end{align}
where $\delta^{(k)}$ equals to 1 if the network has Skip connection and $k = mK_2 + 2$, otherwise it equals to 0
\end{corollary}
\begin{proof}
\begin{align}
& f^{(k)}(w,b,x)-
f^{(k)}(w',b',x)\nonumber 
\\
& =\sigma(\mathrm{Conv}(f^{(k-1)}(w,b,x),w^{(k)})+ g(b^{(k)}))
-\sigma(\mathrm{Conv}(f^{(k-1)}(w,b,x),w'^{(k)}))+ g(b'^{(k)}))\nonumber 
\\
&
+\sigma(\mathrm{Conv}(f^{(k-1)}(w,b,x),w'^{(k)}) + g(b'^{(k)}))
-\sigma(\mathrm{Conv}(f^{(k-1)}(w',b',x),w'^{(k)}) + g(b'^{(k)})).  \label{eq:3331}
\end{align}
From definition of $\mathrm{Conv}()$, the following equation holds.
\begin{align}
\|\mathrm{Conv}(f^{(k-1)}(w,b,x),w^{(k)})_j\| &\leq \sum_{i = 1}^{H^{k - 1}} \|f^{(k-1)}(w,b,x)_i\||(\sum_{p = 1}^{3}\sum_{q = 1}^{3} w^{(k)}_{pqij})|_1 \nonumber \\
                                                           &\leq 9\|f^{(k-1)}(w,b,x)\| \|w_{:,:,:,j}\|\label{eq:3332}
\end{align}
By using lemma\ref{lemma:222}, (\ref{eq:3331}) and (\ref{eq:3332}), corollary\ref{corollary:333} is proved.
\end{proof}

\begin{lemma} \label{lemma:444}
For arbitrary $w,b,x$, 
\begin{align}
 \|f^{(k)}(w,b,x) \|
&\leq {\cal D}_k\|w^{(k)}\|\|w^{(k-1)}\|\cdots \|w^{(2)}\|\|x\|\\  \nonumber
&+{\cal D}_0\|b^{(k)}\|
+ \sum_{j=1}^{k-2} {\cal D}_{j}\|w^{(k)}\|\|w^{(k-1)}\|\cdots \|w^{(k-j)}\|\|b^{(k-j)}\|.
\end{align}
where ${\cal D}_j , 0\leq j \leq k$ is constant.
\end{lemma}
\begin{proof} By considering the case all the parameters of $w'$ and $b'$ are 0, 
in Lemma \ref{lemma:333}, it follows that 
\begin{align}
 \|f^{(k)}(w,b,x) \|
&\leq 9\|w^{(k)}\|\|f^{(k-1)}(w,b,x)\| +L_1L_2\|b^{(k)}\|\\ \nonumber
&+ \delta^{(k)}\|w^{(k)}\|\|f^{(k - K_2 -1)}(w,b,x)-f^{(k - K_2 -1)}(w',b',x)\|.
\end{align}
Then mathematical induction gives the Lemma. 
\end{proof}

\subsection{Notations of parameters}
In order to prove the main theorem, we need several notations. We divide the filters of learning model in each convolutional layer $1 \leq h^{(k)} \leq H^{(k)}$ into the $1 \leq h^{(k)} \leq (H^*)^{(k)}$ and $(H^*)^{(k)} + 1 \leq h^{(k)} \leq H^{(k)}$. The former is denoted as $A$ and the later is denoted as $B$.
The convergent tensor ${\cal E}^{(k)} \in \RR^{3 \times 3 \times H^{(k-1)} \times H^{(k)}}$ and vector ${\cal E}_0^{(k)} \RR^{H^{(k)}}$ 
where the absolute value of all elements are smaller than $1 / \sqrt{n}$ are denoted by
\begin{align}
{\cal E}^{(k)}_{pq} &=
\left(\begin{array}{cc}
{\cal E}_{pqAA}^{(k)} &{\cal E}_{pqAB}^{(k)} 
\\
{\cal E}_{pqBA}^{(k)}  & {\cal E}_{pqBB}^{(k)} 
\end{array}\right),\;\;\; (1 \leq p \leq 3 , 1 \leq q \leq 3), \\
{\cal E}_0^{(k)} &=
\left(\begin{array}{c}
{\cal E}_{A0}^{(k)} 
\\
{\cal E}_{B0}^{(k)} 
\end{array}\right).
\end{align}

The positive constant tensor ${\cal M}^{(k)}$ and vector ${\cal M}_0^{(k)}$ are defined by the condition that
all elements are in the inverval $[A,B]$, 
\begin{align}
{\cal M}^{(k)}_{pq} &=
\left(\begin{array}{cc}
{\cal M}_{pqAA}^{(k)} &{\cal M}_{pqAB}^{(k)} 
\\
{\cal M}_{pqBA}^{(k)}  & {\cal M}_{pqBB}^{(k)} 
\end{array}\right),\;\;\; (1 \leq p \leq 3 , 1 \leq q \leq 3), \\
{\cal M}_0^{(k)} &=
\left(\begin{array}{c}
{\cal M}_{A0}^{(k)} 
\\
{\cal M}_{B0}^{(k)} 
\end{array}\right).
\end{align}
To prove Theorem \ref{theorem:111} \ref{theorem:112}, we show an upper bound of
 $\EE[F_n]$ is given 
by choosing a set $W_E$ which consists of essential weight and bias parameters in Convlutional Layers and Fully connected layers. 

\subsection{No Skip Connection Case}

\noindent{\bf Definition}. (Essential parameter set $W_E$ without Skip Connection). 
A parameter $(w,b)$ is said to be in an essential parameter set $W_E$ 
if it satisfies the following conditions (1),(2) for $2 \leq k \leq K_1$,  \\

(1) For $2 \leq k \leq K^{*}_1$
\begin{align}
w^{(k)}_{pq}&=
\left(\begin{array}{cc}
(w^*)^{(k)}+{\cal E}_{pqAA}^{(k)} & {\cal M}_{pqAB}^{(k)}
\\
-{\cal M}_{pqBA}^{(k)}  & -{\cal M}_{pqBB}^{(k)} 
\end{array}\right),\label{eq:def11}
\\
b^{(k)}&=
\left(\begin{array}{c}
(b^*)^{(k)}+ {\cal E}_{A0}^{(k)} 
\\ 
-{\cal M}_{B0}^{(k)} 
\end{array}\right),\label{eq:def12}
\end{align}
for $1 \leq p \leq 3, 1 \leq q \leq 3$

(2) For $K^{*}_1 + 1 \leq k \leq K_1$
\begin{align}
w^{(k)}_{pq}&=
\left(\begin{array}{cc}
{\cal Z}_{pqAA}^{(k)} & {\cal M}_{pqAB}^{(k)}
\\
-{\cal M}_{pqBA}^{(k)}  & -{\cal M}_{pqBB}^{(k)} 
\end{array}\right),\label{eq:def13}
\\
b^{(k)}&=
\left(\begin{array}{c}
(b^*)^{(k)}+ {\cal E}_{A0}^{(k)} 
\\ 
-{\cal M}_{B0}^{(k)} 
\end{array}\right),\label{eq:def14}
\end{align}
where
\begin{align}\label{eq:casebycase2}
 {\cal Z}_{pqAA}^{(k)}
 =
\left\{\begin{array}{cc}
I_{22AA} + {\cal E}^{(k)}_{22AA}&(p = q = 2)
\\ 
{\cal E}^{(k)}_{pqAA} &(\mathrm{others})
\end{array}\right..
\end{align}
where $I_{22AA} \in \RR^{(H^*)^{(k)}} \times \RR^{(H^*)^{(k)}}$ is an identity matrix.

\begin{lemma} \label{lemma:555}
Assume that the weight and bias parameters of Convolutional layers are in 
the essential set $W_E$ in case without Skip Connection. Then 
there exist constants $c_1,c_2>0$ such that 
\begin{align}
\|f_{:,:,A}^{(K_1)}(w,b,x)-f^{(K_1^*)}(w^*,b^*,x)\|&\leq 
\frac{c_1}{\sqrt{n}}(\|x\|+1),\label{eq:lemma5551}
\\
\|f_{:,:,A}^{(K_1)}(w,b,x)\|&\leq  c_2(\|x\|+1). \label{eq:lemma5552}
\end{align}
\end{lemma}

\begin{proof} 
Eq.\eqref{eq:lemma5552} is derived from Lemma \ref{lemma:444}. 
By the definitions \eqref{eq:def11}, \eqref{eq:def12}, 
for $2\leq k\leq K^*$
\begin{align}
f_A^{(2)}(w,b,x)
&=\sigma(\mathrm{Conv}(f_{:,:,A}^{(1)}(w,b,x), (w^*)^{(2)}+{\cal E}_{:,:,AA}^{(2)}) +
g((b^*)^{(2)}+{\cal E}_{A0}^{(2)})),
\\
f_A^{(k)}(w,b,x)
&=\sigma(\mathrm{Conv}(f_{:,:,A}^{(k-1)}(w,b,x),(w^*)^{(k)}+{\cal E}_{:,:,AA}^{(k)})
\nonumber
\\
&+\mathrm{Conv}(f_{:,:,B}^{(k-1)}(w,b,x), {\cal M}_{:,:,AB}^{(k)})+
g((b^*)^{(k)}+{\cal E}_{A0}^{(k)})). 
\end{align}
In $k = 2$, $|x|$ is bounded and ${\cal M}_{:,:,AB}^{(k)}$ is constant tensor, ${\cal M}_{B0}^{(k)}$ is large sufficiently, $f_{:,:,B}^{(2)}(w,b,x)=0$ because all the elements of the output of ReLU function $f^{(2)}(w,b,x)$ is nonnegative.    For $3\leq k\leq K_1$, $f_{:,:,B}^{(k)}(w,b,x)=0$, since all elements of 
$w_{:,:,BA}^{(k)}$, $w_{:,:,BB}^{(k)}$, and $w_{B0}^{(k)}$ are negative.
Hence by Lemma \ref{lemma:333}, for $2\leq k\leq K_1^*$,
\begin{align}
& \|f_{:,:,A}^{(k)}(w,b,x) - f^{(k)}(w^*,b^*,x) \| \nonumber \\
&\leq 9\|{\cal E}_{:,:,AA}^{(k)}\|\|f^{(k-1)}(w,b,x)\| + L_1 L_2  \|{\cal E}_{A0}^{(k)}\| \nonumber \\
&+ 9\|(w^*)^{(k)}\|\|f^{(k-1)}_{:,:,A}(w,b,x)-f^{(k-1)}(w^*,b^*,x)\|. \label{eq:lemma5553}
\end{align}
and for $K_1^* + 1 \leq  k \leq K_1$, by using $f^{(K_1^*)}(w^*,b^*,x)$ as $f^{(k)}(w^*,b^*,x)$,
\begin{align}
& \|f_{:,:,A}^{(k)}(w,b,x) - f^{(K_1^*)}(w^*,b^*,x) \| \nonumber \\
&\leq 9\|{\cal E}_{:,:,AA}^{(k)}\|\|f^{(k-1)}(w,b,x)\| + L_1 L_2  \|{\cal E}_{A0}^{(k)}\| \nonumber \\
&+ 9\|(w^*)^{(k)}\|\|f^{(k-1)}_{:,:,A}(w,b,x)-f^{(K_1^*)}(w^*,b^*,x)\|. \label{eq:lemma5554}
\end{align}

The elements of tensors and vectors in ${\cal E}_{:,:,AA}^{(k-1)}$ and ${\cal E}_{:,:,A0}^{(k)}$ 
are bounded by $1/\sqrt{n}$ order term, hence
$\|{\cal E}_{AA}^{(k-1)}\|$ and $\|{\cal E}_{A0}^{(k)}\|$ are bounded by 
$1/\sqrt{n}$ order term. Moreover $\|(w^*)^{(k)}\|$ is a constant term. 
For $k=2$, $f_{:,:,A}^{(k-1)}(w,b,x)-f^{(k-1)}(w^*,b^*,x)=x-x=0$. Then, by using mathematical
induction for (\ref{eq:lemma5553}) and (\ref{eq:lemma5554}) , the all terms can be bounded by $1/\sqrt{n}$ terms, hence we obtained the Lemma.
\end{proof}

From \cite{Nagayasu2023a}, because of the output in $k = K_1 + 1$ is nonnegative there exists the essential parameters for fully connected layers such that the number of the convergent parameters${\cal E}$ equals to that of data generating network. From these lemmas, the main theorem can be proved.

(Proof of Theorem \ref{theorem:111}). 
By Lemma \ref{lemma:111}, it is sufficient to prove
that there exists a constant $C>0$ such that 
\begin{align}
\int_{W_E}\exp(-nK(w,b))\varphi(w,b)\mathrm{d}w\mathrm{d}b 
\geq \frac{C}{n^{\lambda}}
\end{align}
From the property of KL-divergence, there exists the positive constant $c_4$
\begin{align}
K(w,b) \leq \frac{c_4}{2}\int \|f^{(K_1 + K_2)}(w,b,x)-f^{(K^*_1 + K^*_2)}(w^*,b^*,x)\|^2 q(x) \mathrm{d}x. 
\end{align}
By using Lemma \ref{lemma:555}, if $(w,b)\in W_E$, 
\begin{align}
K(w,b)\leq \frac{c_4c_3^2}{2n}\int (\|x\|+1)^2  q(x) \mathrm{d}x =\frac{c_5}{n}<\infty.
\end{align}
It follows that 
\begin{align}
&\int_{W_E}\exp(-nK(w,b))\varphi(w,b)\mathrm{d}w\mathrm{d}b \nonumber
\\
& \geq \exp(-c_5) \left(\min_{(w,b)\in W_E} \varphi(w,b)\right) \mbox{Vol}(W_E). \label{proof:1111}
\end{align}
where $c_5>0$, 
$\min_{(w,b)\in W_E}\varphi(w,b) >0$, and $\mbox{Vol}(W_E)$ is the volume 
of the set $W_E$ by the Lebesgue measure. The convergent scale of $\mbox{Vol}(W_E)$ is determined from the number of convergent parameter ${\cal E}$ in $W_E$. 
Then,  
\begin{align}
\mbox{Vol}(W_E)\geq \frac{C_1}{n^{\lambda}}, \label{proof:1112}
\end{align}
where 
\begin{align}
\lambda &= \frac{1}{2}
\left(
 \sum_{k = 2}^{k = K_1} (9H^*_{k - 1} + 1)H^*_{k} + \sum_{k = K_1 + 1}^{k = K_1 + K_2} (H^*_{k - 1} + 1)H^*_{k}
\right) \\ \nonumber
&= \frac{1}{2}
\left(|w^*|_0 + |b^*|_0 + \sum_{k = K^*_1 + 1}^{K_1}(9H_{K^*_1} + 1)H_{K^*_1}
\right).
\end{align}
We obtained theorem\ref{theorem:111}.
\subsection{Skip Connection Case}

\noindent{\bf Definition}. (Essential parameter set $W_E$ with Skip Connection). 
An essential parameter set $W_E$ with Skip Connection satisfies the following conditions (1),(2) for $2 \leq k \leq K_1$,  \\
(1) For $2 \leq k \leq K^{*}_1$, the same conditions as (\ref{eq:def11}) and (\ref{eq:def12}).

(2) For $K^{*}_1 + 1 \leq k \leq K_1$
\begin{align}
w^{(k)}_{pq}&=
\left(\begin{array}{cc}
-{\cal M}_{pqAA}^{(k)} & -{\cal M}_{pqAB}^{(k)}
\\
-{\cal M}_{pqBA}^{(k)}  & -{\cal M}_{pqBB}^{(k)} 
\end{array}\right),\label{eq:def13}
\\
b^{(k)}&=
\left(\begin{array}{c}
-{\cal M}_{A0}^{(k)} 
\\ 
-{\cal M}_{B0}^{(k)} 
\end{array}\right),\label{eq:def14}
\end{align}

\begin{lemma} \label{lemma:777}
Assume that the weight and bias parameters of Convolutional layers are in 
the essential set $W_E$ in case with Skip Connection. Then 
there exist constants $c_1,c_2>0$ such that 
\begin{align}
\|f_{:,:,A}^{(K_1)}(w,b,x)-f^{(K^*_1)}(w^*,b^*,x)\|&\leq 
\frac{c_1}{\sqrt{n}}(\|x\|+1),\label{eq:lemma7771}
\\
\|f_{:,:,A}^{(K_1)}(w,b,x)\|&\leq  c_2(\|x\|+1). \label{eq:lemma7772}
\end{align}
\end{lemma}

\begin{proof} 
Because of similar reason to lemma\ref{lemma:555}, \label{eq:lemma7772}holds.
By Lemma \ref{lemma:333}, for $k = mK_s + 1$, 
\begin{align}
& \|f_{:,:,A}^{(k)}(w,b,x) - f^{(k)}(w^*,b^*,x) \| \nonumber \\
&\leq 9\|{\cal E}_{:,:,AA}^{(k)}\|\|f^{(k-1)}(w,b,x)\| + L_1 L_2  \|{\cal E}_{A0}^{(k)}\| \nonumber \\
&+ 9\|(w^*)^{(k)}\|\|f^{(k-1)}_{:,:,A}(w,b,x)-f^{(k-1)}(w^*,b^*,x)\| \nonumber \\ 
&+ \|w^{(k)}\|\|f^{(k - K_2 -1)}(w,b,x)-f^{(k - K_2 -1)}(w',b',x)\|.\label{eq:lemma7771}
\end{align}
If $k \neq mK_s + 1$ and $2 \leq k \leq K^*_1$,  inequality (\ref{eq:lemma5553}) holds.  Same as the lemma\ref{lemma:555}, from mathematical induction, $\|f_{:,:,A}^{(K^*_1)}(w,b,x) - f^{(K^*_1)}(w^*,b^*,x) \|$ is bounded by $1/\sqrt{n}$ terms. For $2\leq k\leq K_1$, $f_{:,:,B}^{(k)}(w,b,x)=0$ same reason as lemma\ref{lemma:555}. For $K^*_1 + 1 \leq k \leq K_1$, since all elements of $w^{(k)}$ and $b^{(k)}$ are negative, the following equations are given.
\begin{align}
f_{:,:,A}^{(k)}(w,b,x) =
\left\{\begin{array}{cc}
f^{(K^*_1)}(w ,b ,x) &(k = nK_s + 1)
\\ 
0 &(\mathrm{others})
\end{array}\right..
\end{align}
Hence, we obtained the Lemma.
\end{proof}
Same as without Skip connection case, by using the result of \cite{Nagayasu2023a} for fully connected layer and inequality(\ref{proof:1111}),(\ref{proof:1112}), 
we obtained theorem\ref{theorem:112}.

\end{document}